\documentclass[preprint,12pt]{elsarticle}
\usepackage[utf8]{inputenc}
\usepackage[T1]{fontenc}
\usepackage[english]{babel}
\usepackage{amsmath,amsthm,amssymb,amsfonts}
\usepackage{xspace}
\usepackage{hyperref}
\usepackage{cleveref}
\usepackage{verbatim}
\usepackage{thm-restate}
\usepackage{tikz}
\usepackage{multirow,hhline}

\usepackage[hmargin=2.5cm,vmargin=3.8cm]{geometry}

\usepackage{bm}

\newcommand{\ETH}{$\mathsf{ETH}$ }
\newcommand{\NP}{$\mathsf{NP}$ }
\newcommand{\XP}{$\mathsf{XP}$ }

\newcommand{\FPT}{$\mathsf{FPT}$ }

\newcommand{\set}[1]{\left\{#1\right\}}

\newcommand{\NPoly}{\textsf{NP} $\subseteq$ \textsf{coNP}$/$\textsf{poly}\xspace}

\newtheorem{theorem}{Theorem}[section]
\crefname{theorem}{theorem}{theorems}
\newtheorem{lemma}[theorem]{Lemma}
\crefname{lemma}{lemma}{lemmas}
\newtheorem{proposition}[theorem]{Proposition}

\crefname{proposition}{proposition}{propositions}
\crefname{result}{result}{results}

\crefname{corollary}{corollary}{corollaries}

\crefname{fact}{fact}{facts}
\newtheorem{observation}[theorem]{Observation}
\crefname{observation}{observation}{observations}
\newtheorem{question}[theorem]{Question}
\crefname{question}{question}{questions}
\newtheorem{claim}[theorem]{Claim}
\crefname{claim}{claim}{claims}

\crefname{note}{note}{notes}

\crefname{conj}{conjecture}{conjectures}

\crefname{definition}{definition}{definitions}

\crefname{remark}{remark}{remarks}

\newcounter{claimcounter}
\numberwithin{claimcounter}{lemma}

\newenvironment{proofofclaim}{%
    
  \proof}{\endproof}

\tikzstyle{noeud}=[circle,inner sep=2, minimum size =3 pt, line width = 1pt, draw=black, fill=white]

\newcommand{\Pb}[4]{%
\begin{center}
  \begin{tabular}{|l|}%
  \hline
    \begin{minipage}[c]{0.95\textwidth}
      \smallskip%
      \par\noindent%
      #1%
      \par\noindent%
      \textbf{\textsf{Input}}: #2% 
      \par\noindent%
      \textbf{\textsf{#3}}: #4 
      \smallskip%
      \par\noindent%
    \end{minipage}
  \\\hline
  \end{tabular}%
\end{center}
}%

\begin{document}
\title{Learning Small Decision Trees with Few Outliers: A Parameterized Perspective}
\author[label1]{Harmender Gahlawat}
%\ead{harmendergahlawat@gmail.com}

\author[label2]{Meirav Zehavi}
%\ead{meiravze@bgu.ac.il}

\affiliation[label1]{organization={Universit\'e Clermont Auvergne, CNRS, Clermont Auvergne INP, Mines Saint-\'Etienne, LIMOS},
            %addressline={},
            city={Clermont-Ferrand},
            postcode={63000},
            %state={},
            country={France}}

\affiliation[label2]{organization={Ben-Gurion University of the Negev},
            %addressline={},
            city={Beersheba},
            %postcode={63000},
            %state={},
            country={Israel}}


\newcommand{\CI}{$\mathsf{CI}$\xspace}
\newcommand{\DT}{$\mathsf{DT}$\xspace}
\newcommand{\DTL}{\textsc{Decision Tree Learning}\xspace}

\newcommand{\DTS}{\textsc{DTS}\xspace}
\newcommand{\DTD}{\textsc{DTD}\xspace}
\newcommand{\DTSO}{\textsc{DTSO}\xspace}
\newcommand{\DTDO}{\textsc{DTDO}\xspace}
\newcommand{\HS}{\textsc{HS}\xspace}
\newcommand{\SC}{\textsc{Set Cover}\xspace}

\newcommand{\IPC}{\textsc{Isometric Path Cover}\xspace}
\newcommand{\RnJ}{\textsc{Romeo and Juliet}\xspace}
\newcommand{\PC}{\textsc{Path Cover}\xspace}
\newcommand{\IPP}{\textsc{Isometric Path Partition}\xspace}
\newcommand{\SGS}{\textsc{Strong Geodetic Set}\xspace}
\newcommand{\GS}{\textsc{Geodetic Set}\xspace}
\newcommand{\PART}[1]{\textsc{#1-Partition}\xspace}
\newcommand{\IPART}[1]{\textsc{Induced #1-Partition}\xspace}
\newcommand{\IPATHCV}{\textsc{Induced Path Cover}\xspace}
\newcommand{\MONOSET}{\textsc{Monophonic Set}\xspace}
\newcommand{\VC}{\textsc{Vertex Cover}\xspace}
\newcommand{\FVS}{\textsc{Feedback Vertex Set}\xspace}

\begin{abstract}
Decision trees are a fundamental tool in machine learning for representing, classifying, and generalizing data. It is desirable to construct ``small'' decision trees, by minimizing either the \textit{size} ($s$) or the \textit{depth} $(d)$ of the \textit{decision tree} (\textsc{DT}). Recently, the parameterized complexity of \textsc{Decision Tree Learning} has attracted a lot of attention.
We consider a generalization of \textsc{Decision Tree Learning} where given a \textit{classification instance} $E$ and an integer $t$, the task is to find a ``small'' \textsc{DT} that disagrees with $E$ in at most $t$ examples. We consider two problems: \textsc{DTSO} and \textsc{DTDO}, where  the goal is to construct a \textsc{DT} minimizing $s$ and $d$, respectively. We first establish that both \textsc{DTSO} and \textsc{DTDO} are W[1]-hard when parameterized by $s+\delta_{max}$ and $d+\delta_{max}$, respectively, where $\delta_{max}$ is the maximum number of features in which two differently labeled examples can differ. We complement this result by showing that these problems become \textsc{FPT} if we include the parameter $t$. We also consider the kernelization complexity of these problems and establish several positive and negative results for both \textsc{DTSO} and \textsc{DTDO}.
\end{abstract}

\maketitle

\section{Introduction}
{\let\thefootnote\relax\footnotetext{A preliminary version of this paper appeared in the proceedings of the 38th Annual AAAI Conference on Artificial Intelligence (AAAI 2024)~\cite{ourAAAIDTL}.}}\par

Decision trees is a fundamental tool in the realm of machine learning with applications spanning classification, regression, anomaly detection, and recommendation systems~\cite{larose2014discovering,murthy1998automatic, quinlan1986induction}. Because of their ability to represent complex labeled datasets through a sequence of simple binary decisions, they provide a highly interactive and interpretable model for data representation~\cite{darwiche2020reasons, doshi2017roadmap,goodman2017european,lipton2018mythos,monroe2018ai}. It is of interest to have \textit{small trees}, as they require fewer tests to classify data and are easily interpretable. Moreover, small trees are expected to generalize better to  new data, i.e., minimizing the number of nodes reduces the chances of overfitting~\cite{bessiere2009minimising}.  However, as  problem instances grow in size and complexity, efficiently learning small decision trees becomes a challenging task. In particular, it is \NP-hard to decide if a given dataset can be represented using a decision tree of a certain size or depth~\cite{laurent1976constructing}. To deal with this complexity barrier, implementations of several heuristics based on constraint-based and SAT-based techniques have been proposed to learn small decision trees~\cite{bessiere2009minimising, narodytska2018learning, avellaneda2020efficient, schidler2021sat}.  In fact, the classical \textsc{CART} heuristic herein is among the top 10 algorithms of data mining chosen by the ICDM~\cite{steinberg2009cart,wu2008top}.

Despite these efforts, our understanding of the computational complexity of learning small decision trees is limited. Recently, research of the parameterized complexity of learning small decision trees has attracted a lot of attention~\cite{eibenLargeDomain,DTGeometry,ordyniak,komusiewicz2023computing}. Parameterized complexity theory provides a framework for classifying computational problems based on both their input size and a parameter that captures the inherent difficulty of the problem. The key notion in parameterized complexity is that of \textit{fixed parameter tractability} (\FPT), which restricts the exponential blowup in the time to be a function only of the chosen \textit{parameter}. Due to its efficacy, parameterized complexity theory has been extensively used to understand the complexity of several problems arising in AI and ML; see, e.g.~\cite{backstrom2012complexity,bessiere2008parameterized,bredereck2017parliamentary,ganian2018parameterized,gaspers2017backdoors}. An important framework within parameterized complexity is that of \textit{kernelization}, polynomial-time preprocessing with a parametric guarantee. Due to its profound impact, kernelization was termed ``the lost continent of polynomial time''~\cite{kernelApplication}. Kernelization is specifically useful in practical applications as it has shown tremendous speedups in practice~\cite{gao2009data, guo2007invitation,niedermeier2000general,weihe1998covering}.% For more details on kernelization and parameterized algorithms, we refer to the books~\cite{bookParameterized, bookKernelization}.

The input to a decision tree learning algorithm is a \textit{classification instance} (\CI) $E$, which is a set of \textit{examples} labeled either positive or negative, where each \textit{example} is defined over the same set of \textit{features} $F$ that can take values from a linearly ordered domain $D$. We define these concepts (along with other definitions) in Section~\ref{S:prelim}. Now, in \DTS  (resp., \DTD), the input is a positive integer $s$ (resp., $d$) and a \CI $E$, and the goal is to decide whether there exists some \textit{decision tree} (\DT)  of \textit{size} at most $s$ (resp., \textit{depth} at most $d$) that ``\textit{classifies}'' each example of $E$ correctly. The \textit{size} of a \DT $T$ is the number of \textit{test nodes} in $T$ and the \textit{depth} of a $T$ is the maximum number of test nodes on a leaf to root path of $T$. Some parameters of the input that are of interest, in addition to $s$ and $d$, are: $|F|$ being the number of features in $E$, $D_{max}$ being the maximum value a feature of an example can take, and $\delta_{max}$ being the maximum number of features a positive and a negative example can differ in. All of these features are often small compared to the size of the input instance and hence are good choices to achieve \FPT algorithms (for e.g., see Table~1 in~\cite{ordyniak}).  

In practice, it is often acceptable to allow for small margins of error. We consider a generalization of \textsc{Decision Tree Learning} where, given a \CI $E$ and an integer $t$, the goal is to compute a \DT $T$ of minimum size or depth such that $T$ disagrees with at most $t$ examples of $E$. Notably, even when we are allowed to have 1 \textit{outlier} (i.e., $t=1$), the instance used to prove W[1]-hardness of \DTS and \DTD~\cite{ordyniak} admits a \DT of size and depth 0. Hence, a little slack (a few outliers) can decrease the complexity of the resulting tree significantly. Moreover, as shown by Ruggieri~\cite{yadt} in the results of YaDT on the benchmark data set, in practice, allowing even a small number of outliers can decrease the size of the optimal decision tree substantially.
We term the corresponding problems for \DTD and \DTS as \DTDO and \DTSO, respectively. (Here, $\mathsf{O}, \mathsf{S}$, and $\mathsf{D}$ signify outlier, size and depth, respectively.) Finally, observe that \DTDO and \DTSO model \DTD and \DTS, respectively, when $t=0$.% It is known that \DTD and \DTS are \FPT when parameterized by $d+\delta_{max}$ and $s+\delta_{max}$, respectively~\cite{}. First, we establish that \DTDO and \DTSO are W[1]-hard parameterized by $d+\delta_{max}$ and $s+\delta_{max}$, respectively, even when $\delta_{max} =3$. Second, we achieve tractability for both \DTDO and \DTSO by including $t$ as a parameter. We also establish several positive and negative results concerning the kernelization complexity of \DTDO and \DTSO. 

\medskip
\noindent\textbf{Previous Work.} The parameterized complexity of learning small decision trees has attracted significant attention recently. Ordyiak and Szeider~\cite{ordyniak} established that both \DTD and \DTS are W[1]-hard parameterized $d$ and $s$, respectively, even for binary instances. On the positive side, they proved that \DTS (resp., \DTD) is \FPT parameterized by $s+\delta_{max}+D_{max}$ (resp., $d+\delta_{max}+D_{max}$). Kobourov et al.~\cite{DTGeometry} established that \DTS is W[1]-hard parameterized by $|F|$,  \XP parameterized by $|F|$, and \FPT when parameterized by $s + |F|$. More recently, Eiben et al.~\cite{eibenLargeDomain} established that both \DTS and \DTD are \FPT parameterized by $s+\delta_{max}$ and $d+\delta_{max}$, respectively. 

\medskip
\noindent\textbf{Our Contribution.}
We study the parameterized and kernelization complexity of \DTDO and \DTSO. The notations used in this section are definded in Section~\ref{S:prelim}. We start by establishing that \DTDO and \DTSO are W[1]-hard parameterized by $d+\delta_{max}$ and $s+\delta_{max}$, respectively. For this purpose, we provide an involved reduction from \textsc{Partial Vertex Cover} to \DTSO and \DTDO.

\begin{restatable}{theorem}{WHard}\label{th:Whard}
\textsc{DTDO} and \DTSO are $W[1]$-hard when parameterized by $d+\delta_{max}$ and $s+\delta_{max}$, respectively. Further, they are $W[1]$-hard parameterized by $s$ and $d$, respectively, even if $\delta_{max} \leq 3$.
\end{restatable}

We complement our W[1]-hardness result with \FPT algorithms for \DTSO and \DTDO by including $t$ as an additional parameter. We build on the algorithms provided by Eiben et al.~\cite{eibenLargeDomain} for \DTS and \DTD parameterized $s+\delta_{max}$ and $d+\delta_{max}$, respectively.

\begin{restatable}{theorem}{FPTT}\label{th:FPT}
\textsc{DTDO} and \DTSO are \FPT parameterized by $d+\delta_{max}+t$ and $s+\delta_{max}+t$, respectively. 
\end{restatable}

Next, we consider the kernelization complexity of \DTSO and \DTDO. Since \DTD and \DTS are the special cases of \DTDO and \DTSO when $t=0$, we prove negative kernelization results for \DTS and \DTD, which imply the same for \DTSO and \DTDO. We first observe that the reduction provided by Ordyniak and Szeider~\cite{ordyniak} from \textsc{Hitting Set} (\HS) to \DTS and \DTD to prove W[1]-hardness is a polynomial parameter transformation. Since \HS is unlikely to admit a polynomial compression parameterized by $k+|\mathcal{U}|$ (Proposition~\ref{P:HS}), we have the following theorem.

\begin{restatable}{observation}{HSIncompressible}\label{th:HSIncompressible}
\DTS and \DTD parameterized by $s+|F|+D_{max}$ and  $d+|F|+D_{max}$, respectively, do not admit a polynomial compression even when $D_{max} =2$, unless \NPoly.
\end{restatable}

Since \HS admits a kernel of size at most $k^{\mathcal{O}(\Delta)}$~\cite{bookKernelization} where $\Delta$ is the \textit{arity} of the \HS instance, it becomes interesting to determine the kernelization complexity of \DTSO and \DTDO when $\delta_{max}$ is a fixed constant (as $\delta_{max}=\Delta$ in the reduction). Towards this, we establish that \DTSO and \DTDO admit trivial polynomial kernels parameterized by $D_{max}+|F|$ when $\delta_{max}$ is a fixed constant. 

\begin{restatable}{theorem}{TrivialKernel}\label{th:trivial}
\DTS and \DTD parameterized by $D_{max}+|F|$ admit a trivial polynomial kernel when $\delta_{max}$ is a constant.
\end{restatable}

Next, we establish the incompressibility  of \DTD parameterized by $d+|F|$ even when  $\delta_{max}$ is a fixed constant. To this end, we provide a non-trivial AND-composition for \DTD parameterized by $d+|F|$ such that $\delta_{max} \leq 3$. 
\begin{restatable}{theorem}{DTDIncompressible}\label{th:dtdIncompressible}
\DTD parameterized by $d+|F|$ does not admit a  polynomial compression even when $\delta_{max}\leq 3$, unless \NPoly.
\end{restatable}

Let $|E|$ denote the number of examples in $E$. Notice that, possibly, $|F|>|E|$. Since, \HS and \textsc{Set Cover} are dual problems, we observe that this duality can be used alongwith the reduction provided by Ordyniak and Szeider~\cite{ordyniak} (from \HS to \DTD and \DTS) to get a polynomial parameter transformation from \textsc{Set Cover} to \DTS and \DTD, which in turn imply their incompressibility parameterized by $|E|$. 
\begin{restatable}{observation}{nIncompressible}\label{th:nIncompressible}
Let $|E|$ be the number of examples in $E$. \DTD and \DTS parameterized by $|E|+D_{max}$ do not admit a  polynomial compression even when $D_{max}\leq 2$, unless \NPoly.
\end{restatable}

% Finally, we note that the reductions from \textsc{Set Cover} and \textsc{HS} to \DTS and \DTD imply the hardness of approximation for both of these problems in \FPT time. 
% \begin{theorem}
    
% \end{theorem}

\section{Preliminaries}\label{S:prelim}
For $\ell\in \mathbb{N}$, let $[\ell]= \{ 1,\ldots, \ell\}$. For the graph theoretic notations not defined explicitly here, we refer to~\cite{diestel}. 

\medskip
\noindent\textbf{Classification Problems.}
An example $e$ is a function $e~:~F(e) \rightarrow D$ defined over a finite set $F(e)$ of \textit{features} and a (possibly infinite) linearly ordered \textit{domain} $D\subset \mathbb{Z}$. 
A \textit{classification instance (CI)} $E= E^+ \bigcup E^-$ is the disjoint union of two sets of \textit{positive examples} $E^+$ and \textit{negative exapmles} $E^-$, defined over the same set of features, i.e., for $e_1, e_2 \in E$, $F(e_1) = F(e_2)$. Here,  $F(E) = F(e)$ for some $e\in E$. A set of examples $X \in E$ is \textit{uniform} if $X\in E^+$ or $X\in E^-$; otherwise, $X$ is \textit{non-uniform}.  When it is clear from the context, we denote $F(E)$ by $F$.
For two examples $e_1, e_2\in E$ and some feature $f\in F$, if $f(e_1) = f(e_2)$, then we say that $e_1$ and $e_2$ \textit{agree} on $f$, else, we  say that $e_1$ and $e_2$ \textit{disagree} on $f$. A subset $S \subseteq F$ is said to be a \textit{support set} of $E$ if any two examples $e^+ \in E^+$ and $e^- \in E^-$ disagree in at least one feature of $S$. It is \NP-hard to compute a support set of minimum size~\cite{ibaraki2011partially}. 

For two examples $e,e'\in E$, let $\lambda(e,e')$ denote the set of features where $e$ and $e'$ disagree. Moreover, let $\lambda_{max}(E) = \max_{e^+\in E^+ \wedge e^- \in E^-} |\lambda(e^+,e^-)|$ denote the maximum number of features any two non-uniform examples disagree on. 

For a feature $f\in F$, let $D_E(f)$ denote the set of domain values appearing in any example of $E$, i.e., $D_E(f) = \set{e(f)~|~e \in E}$. Moreover, let $D_{max}$ denote the maximum size of $D_E(f)$ over all features of $E$, i.e., $D_{max}= \max_{f\in F} |D_E(f)|$. If $D_{max} = 2$, then the classification instance is said to be \textit{Boolean}. 
Let $E[\alpha]$ denote the set of examples that agree with the assignment $\alpha~:~ F' \rightarrow D$, where $F'\subseteq F(E)$, i.e., $E[\alpha] = \set{e~|~ e(f) = \alpha(f) \wedge f\in F'}$. 

\medskip
\noindent\textbf{Decision Tree.}
A \textit{decision tree} $(\mathsf{DT})$ is a rooted tree $T$, with vertex set $V(T)$ and arc set $A(T)$, where each non-leaf node $v\in V(T)$ is labeled with a \textit{feature} $f(v)$ and an integer \textit{threshold} $\lambda(v)$, each non-leaf node has exactly two outgoing arcs, a \textit{left arc} and a \textit{right arc}, and each leaf is either a \textit{positive leaf} or a  \textit{negative leaf}. Let $F(T) = \set{f(v)~|~v\in V(T)}$. A non-leaf node of $T$ is referred to as a \textit{test node}. For $v\in V(T)$, let $T_v$ denote the subtree of $T$ rooted at $v$.

Consider a \CI $E$ and \DT $T$ with $F(T) \subseteq F(E)$. For each node $v\in T$, let $T_E(v)$ be the set of all examples $e\in E$ such that for each left (resp., right) arc $uw$ on the unique path from the root of $T$ to $v$, we have $e(F(u)) \leq \lambda(u)$ (resp., $e(F(u)) > \lambda(u)$). $T$ \textit{correctly classifies} an example $e\in E$ if $e$ is a positive (resp. negative) example and $e\in T_E(v)$ for a poistive (resp., negative) leaf $v$. Similarly, $T$ \textit{correctly classifies} an instance $E$ if $T$ correctly classifies every example of $E$. Here, we also say that $T$ is a \DT for $E$. The \textit{size} of $T$, denoted by $|T|$, is the number of test nodes in $T$. Similarly, the \textit{depth} of $T$, denoted by $dep(T)$, is the maximum number of test nodes in any root-to-leaf path of $T$. \DTS (resp., \DTD) is now the  problem of deciding whether given \CI $E$ and a natural number $s$ (resp., $d$), is there a \DT of size at most $s$ (resp., depth at most $d$) that classifies $E$. More formally, we have the following decision problems.

\medskip

% \noindent\fbox{\parbox{80mm}{
%     \noindent
%     \underline{\DTS}\\
%     \textbf{Input:} A classification instance $E$, and an integer $s \in \mathbb{N}$.\\
%     \noindent
%     \textbf{Question:} Is there a \DT $T$ of size at most $s$ that classifies $E$?
% }
% }
\Pb{\DTS}{A classification instance $E$, and an integer $s \in \mathbb{N}$.}{Question}{Is there a \DT $T$ of size at most $s$ that classifies $E$?}

% \noindent\fbox{\parbox{80mm}{
%     \noindent
%     \underline{\DTD}\\
%     \textbf{Input:} A classification instance $E$, and an integer $d \in \mathbb{N}$.\\
%     \noindent
%     \textbf{Question:} Is there a \DT $T$ of depth at most $d$ that classifies $E$?
% }
% }

\Pb{\DTD}{A classification instance $E$, and an integer $d \in \mathbb{N}$.}{Question}{Is there a \DT $T$ of depth at most $d$ that classifies $E$?}

\medskip
\noindent \textbf{Decision Trees with Outliers.}
Let $E$ be a \CI, and $T$ be a \DT corresponding to $E$, which does not necessarily classifies $E$. Let $v$ be a positive (resp., negative) leaf of $T$. Then, we say that an example $e \in E$ is an \textit{outlier for $T$} if $e \in E^+$ (resp., $e\in E^-$) and $e\in T_E(v)$ for some negative (resp., positive leaf) $v$. We just say that $e$ is an outlier when it is clear from the context. Let $O(T,E)$ denote the set of outliers in $E$ for $T$.  Then, $|O(T,E)|$ is the number of outliers for $T$ in $E$.
\DTSO (resp., \DTDO) is now the  problem of deciding whether given \CI $E$ and natural number $s$ and $t$ (resp., $d$ and $t$), is there a \DT of size at most $s$ (resp., depth at most $d$) such that $|O(T,E)|\leq t$.
Consequently, we have the following decision problems.

% \medskip

% \noindent\fbox{\parbox{80mm}{
%     \noindent
%     \underline{\DTSO}\\
%     \textbf{Input:} A classification instance $E$, and integers $s, t \in \mathbb{N}$.\\
%     \noindent
%     \textbf{Question:} Is there a \DT $T$ of size at most $s$ that has at most $t$ outliers for $E$?
% }
% }

\Pb{\DTSO}{A classification instance $E$, and integers $s, t \in \mathbb{N}$.}{Question}{Is there a \DT $T$ of size at most $s$ that has at most $t$ outliers for $E$?}

% \noindent\fbox{\parbox{80mm}{
%     \noindent
%     \underline{\DTDO}\\
%     \textbf{Input:} A classification instance $E$, and integers $d,t \in \mathbb{N}$.\\
%     \noindent
%     \textbf{Question:} Is there a \DT $T$ of depth at most $d$ that has at most $t$ outliers for $E$?
% }
% }

\Pb{\DTDO}{A classification instance $E$, and integers $d, t \in \mathbb{N}$.}{Question}{Is there a \DT $T$ of depth at most $d$ that has at most $t$ outliers for $E$?}

\noindent \textbf{Parameterized Complexity.} 
% \subsection{Parameterized complexity}
In the framework of parameterized complexity, each problem instance is associated with a non-negative integer, called a \textit{parameter}. A parametrized problem $\Pi$ is \textit{fixed-parameter tractable} ($\mathsf{FPT}$) if there is an algorithm that, given an instance $(I,k)$ of $\Pi$, solves it in time $f(k)\cdot |I|^{\mathcal{O}(1)}$ for some computable function $f(\cdot)$. Central to parameterized complexity is the W-hierarchy of complexity classes:
$
\mathsf{FPT} \subseteq \mathsf{W[1]}  \subseteq \mathsf{W[2]} \subseteq \ldots \subseteq \mathsf{XP}.
$ Specifically, \FPT $\neq$ W[1], unless \ETH fails.

Two instances $I$ and $I'$ are \textit{equivalent} when $I$ is a Yes-instance iff $I'$ is a Yes-instance. A \textit{compression} of a parameterized problem $\Pi_1$ into a (possibly non-parameterized) problem $\Pi_2$ is a polynomial-time algorithm that maps each instance $(I,k)$ of $\Pi_1$ to an equivalent instance $I'$ of $\Pi_2$ such that size of $I'$ is bounded by $g(k)$ for some computable function $g(\cdot)$. If $g(\cdot)$ is polynomial, then the problem is said to admit a \textit{polynomial compression}.
A \textit{kernelization algorithm} is a compression where $\Pi_1 = \Pi_2$. Here, the output instance is called a \textit{kernel}.  
Let $\Pi_1$ and $\Pi_2$ be two parameterized problems. A \textit{polynomial parameter transformation} from $\Pi_1$ to $\Pi_2$ is a polynomial-time algorithm that, given an instance $(I,k)$ of $\Pi_1$, generates an equivalent instance $(I',k')$ of $\Pi_2$
such that $k' \leq p(k)$, for some polynomial $p(\cdot)$. It is well-known that if $\Pi_1$ does not admit a polynomial compression, then $\Pi_2$ does not admit a polynomial compression~\cite{bookParameterized}. An \textit{AND-composition} from a problem $P$ to a parameterized problem $Q$ is an algorithm that takes as input $N$ instances $I_1,\ldots,I_N$ of $P$, and in time polynomial in $\sum_{j\in [N]} |I_j|$, outputs an instance $(\mathcal{I},k)$ of $Q$ such that: (1) $(\mathcal{I},k)$ is a Yes-instance iff  $I_j$ is a Yes-instance for each $j\in [N]$, and (2) $k$ is bounded by a polynomial in $\max_{j\in [N]} |I_j| + \log N$. It is well known that if $P$ is NP-hard, then $Q$ does not admit a polynomial compression parameterized by $k$, unless \NPoly~\cite{bookKernelization}.  We refer to the books~\cite{bookParameterized,bookKernelization} for  details on parameterized complexity.

\medskip
\noindent\textbf{Hitting Set, Set Cover, and Decision Trees.}
Given a family of sets $\mathcal{F}$ over some universe $U$ and an integer $k$, the \textsc{Hitting Set} (\HS) problem asks whether $\mathcal{F}$ has a \textit{hitting set} of size $k$, i.e., a subset $H$ of $U$ of size at most $k$ such that $X\cap H \neq \emptyset$ for every $X\in \mathcal{F}$. The \textit{maximum} \textit{arity} $\Delta$ of a \HS instance is the size of a largest set in $\mathcal{F}$. Ordyniak and Szeider~\cite{ordyniak} provided the following reduction from \HS to \DTS and \DTD:  For an instance $\mathcal{I} = (\mathcal{F}, U,k)$ of \HS, let $E(\mathcal{I})$ be the \CI that has a (Boolean) feature for every element in $U$, one positive example $p$ with $p(u) = 0$ for every $u\in U$; and one negative example $n_X$ for every $X\in \mathcal{F}$ such that $n_X(u)=1$ for every $u\in X$ and $n_X(u)=0$, otherwise. It is easy to see here that $\delta_{max}$ of $E(\mathcal{I)}$ is the same as the $\Delta$ (arity) of $\mathcal{I}$ and $D_{max} = 2$. 
Similarly to above construction, we define $\overline{E}(\mathcal{I})$ by turning each positive example to negative example and each negative example to positive, i.e., $p \in \overline{E}^-(\mathcal{I})$ and $n_X \in \overline{E}^+(\mathcal{I})$, for $X\in \mathcal{F}$. Observe that $E(\mathcal{I})$ admits a \DT of size $s$ (resp., depth $d$) iff $\overline{E}(\mathcal{I})$ admits a \DT of size $s$ (resp., depth $d$). Ordyniak and Szeider~\cite{ordyniak} proved the following result.

\begin{proposition}[\cite{ordyniak}]
    $\mathcal{I}$ has a hitting set of size at most $k$ iff $E(\mathcal{I})$ (resp. $\overline{E}(\mathcal{I})$) admits a \DT of depth (resp., size) at most $k$.
\end{proposition}

In an instance of \SC $\mathcal{I} = (\mathcal{F}, U, k)$, we are given a family of sets $\mathcal{F}$ over some universe $U$ and the problem asks whether there exist $k$ sets $X_1,\ldots, X_k \in \mathcal{F}$ such that $\bigcup_{i\in [k]} X_i = U$. Since \HS and \SC are dual problems, it is not surprising that we have the following construction, which is similar to the construction in the reduction from \HS to \DTS and \DTD: Let $\mathcal{I} = (\mathcal{F}, U, k)$ be an instance of \SC. We construct the following boolean \CI $E(\mathcal{I})$. Here, the set $F(E(\mathcal{I}))$ (feature set of $E(\mathcal{I})$) corresponds to $\mathcal{F}$ and each negative example corresponds to an element of $U$; additionally, we have a positive dummy example. More formally, we have a positive example $p$ with $p(X) = 0$ for every $X\in \mathcal{F}$; and one negative example $n_u$ for every $u\in U$ such that $n_u(X) = 1$ if $u\in X$ and $n_u(X) =0$, otherwise. Here, observe that $D_{max} = 2$. We have the following straightforward observation.
\begin{observation}\label{L:SC}
  $\mathcal{I}$ has a set cover of size $k$ iff $E(\mathcal{I})$ admits a \DT of depth (resp., size) at most $k$.  
\end{observation}

It is well known that \HS and \SC parameterized by $k+|\mathcal{U}|$ are unlikely to admit a polynomial kernel:

\begin{proposition}[\cite{RBDSIncompressibility}] \label{P:HS}
    \HS and \SC parameterized by $k+|U|$ do not admit a polynomial compression, unless \NPoly.
\end{proposition}
 
Observe that the reductions provided above from \HS and \SC to \DTS and \DTD are polynomial parameter transformations, which implies the results from Observations~\ref{th:HSIncompressible} and \ref{th:nIncompressible}. 
% \begin{proof}
%    First, let $\mathcal{I}$ has a  
% \end{proof}

%Lemma~\ref{L:SC} has interesting implications. One of them is that it implies the \FPT-inapproximability of learning \textsc{DT}s of small size. We discuss further implications later.

\section{Hardness for \DTSO and \DTDO}
It was recently established that \DTS (resp., \DTD) is \FPT when parameterized by $s+\delta_{max}$ (resp., $d+ \delta_{max}$)~\cite{eibenLargeDomain}. In this section, we establish that $\DTSO$ (resp., \DTDO) is W[1]-hard when parameterized by $s+\delta_{max}$ (resp., $d+\delta_{max}$). For this purpose, we first define the problem \textsc{Partial Vertex Cover} (\textsc{PVC}). In \textsc{PVC}, given a graph $G$ and integers $k,p \in \mathbb{N}$, the goal is to decide if there is a subset $U\subseteq V(G)$ such that $|U| \leq k$ and $|\{uv~|~ uv\in E(G) \wedge \{u,v\} \cap U \neq \emptyset)\}| \geq p$. \textsc{PVC} is W[1]-hard when parameterized by $k$~\cite{pvcHard}:

\begin{proposition}[\cite{pvcHard}]\label{P:pvcHard}
    \textsc{PVC} parameterized by $k$ is W[1]-hard.
\end{proposition}

% \medskip
% \noindent\fbox{\parbox{80mm}{
%     \noindent
%     \underline{\textsc{PVC}}\\
%     \textbf{Input:} A graph $G$, and integers $k,p \in \mathbb{N}$.\\
%     \noindent
%     \textbf{Question:} Is there a subset $U\subseteq V(G)$ such that $|U| \leq k$ and $|\{uv~|~ uv\in E(G) \wedge \{u,v\} \cap U \neq \emptyset)\}| \geq p$?
% }
% }
%  \medskip

%\textcolor{red}{I am using the following facts. I will write this paragraph in a nice manner.} 
For the purpose of this section, let $m$ denote $|E(G)|$ and $n$ denote $|V(G)|$. Now, we provide a construction that we will be using to prove W[1]-hardness for \DTSO and \DTDO.
\begin{figure}
    \centering
    \includegraphics[scale=0.8]{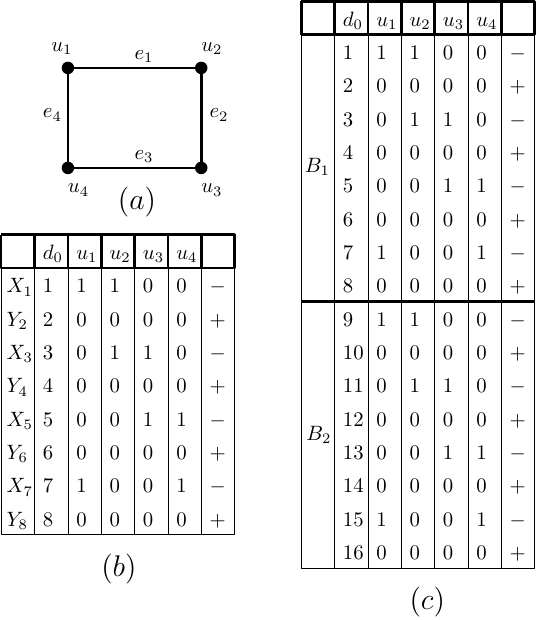}
    \caption{Here $(a)$ is the graph $G$ and $(b)$ is the \CI $E'$ corresponding to $G$. In $(c)$, we illustrate an example for constructing $E$ from $E'$ for $\eta =2$.}
    \label{fig:WHard}
\end{figure}

\subsection{Construction.} Let $(G,k,p)$ be an instance of \textsc{PVC}. Fix an ordering $e_1,\ldots, e_m$ of the edges of $G$. We consider the following \CI $E'$ with feature set $F= V(G) \cup \{d_0\}$. See Figure~\ref{fig:WHard} for a reference. Formally, for every vertex $u\in V(G)$, we have a feature $u$ in $F$, along with a \textit{dummy feature} $d_0$. The features corresponding to vertices in $V(G)$ are said to be \textit{vertex features}. For each edge $e_i\in E(G)$, we add one negative example $X_{2i-1}$ such that $X_{2i-1}(d_0) = 2i-1$ and  for a vertex feature $v$, $X_{2i-1}(v) = 1$ if $v$ is an endpoint of $e_i$, and $X_{2i-1}(v) = 0$, otherwise. Moreover, we add $m$ positive examples in the following manner. For $i\in [m]$, we add a positive example $Y_{2i}$ such that $Y_{2i}(d_0) = 2i$ and for each vertex feature $v$, $Y_{2i}(v) = 0$. 

Now, we make a \CI $E$ by taking $\eta$ copies of examples in $E'$ in the following manner. The value of $\eta$ will be fixed later for the proofs of \DTSO and \DTDO separately. Let $\ell = m-p$, and $F(E) =F(E') = V(G) \cup \{d_0\}$. To make each example coming from some copy of $E'$ unique, we will set the values of the  feature $d_0$ in increasing order. More formally, let $E'_1, \ldots E'_\eta$ be $\eta$ copies of the \CI $E'$. Now, for $i\in [\eta]$, if $e'\in E'_i$, then we add an example $e\in E$ such that $e(d_0) = 2m(i-1)+e'(d_0)$ and for every vertex feature $v\in V(G)$, $e(v)=e'(v)$. Let $B_i$, later referred to as \textit{block} $B_i$, be the set of examples in $E$ corresponding to the examples in the copy $E'_i$ of $E$.  Finally, we set $s=d=k$ and $t =\ell\eta$. This completes our construction.

\subsection{Some Preliminaries and Observations.}
Each vertex feature in $F(E)$ corresponds to a unique vertex in $V(G)$. Depending on the feature $f(v)$ of a test node $v$, we categorize the test nodes in two categories: \textit{vertex test node} where $f(v)$ is a vertex feature or a \textit{dummy test node} where $f(v) = d_0$. The positive examples in $E$ ($E^+$) are said to be \textit{dummy examples} and the negative examples in $E$ ($E^-$) are said to be \textit{edge examples}; each edge example corresponds to a unique edge, and each edge corresponds to exactly $\eta$ edge examples. We say that an edge example $e\in E$, corresponding to an edge $e'\in E(G)$, is \textit{hit} by a vertex feature $u$ if $u$ is an endpoint of $e'$. Accordingly, let $H(u)$ denote the set of all edge examples hit by the vertex feature $u$.

Since we want to learn ``small'' decision trees, we assume that our \DT $T$ has the following properties. For each test node $v$, $T_E(v)$ are non-uniform; otherwise, we can replace $T_v$ by a leaf node $v'$ to get a ``smaller'' \DT with the same classification power. Similarly, for each test node $v$ with children $l$ and $r$, $T_E(l) \neq \emptyset$ and $T_E(r) \neq \emptyset$, otherwise we can get a smaller \DT with the same classification power. To see this, let $T_E(l)= \emptyset$, then $T'$ obtained by replacing $T_v$ by $T_r$ is the desired \DT. We have the following observation.

\begin{observation}\label{O:parsimonious}
    Let $v$ be a vertex test node in a \DT $T$ for $E$, and $f(v) = x$. Moreover, let $l$ and $r$ be the left and right child of $v$, respectively. Then, $\lambda(v) = 0$, $r$ is a negative leaf, and $T_E(r)$ contains only negative examples. More specifically, $T_E(r)$ contains all of the edge examples in $T_E(v)$ that are hit by the vertex $x$.
\end{observation}
\begin{proof}
First, targeting a contradiction, assume that $\lambda(v) \neq 0$. If $\lambda(v) > 0$ (resp., $\lambda(v) <0$), then observe that $T_E(r) = \emptyset$ (resp., $T_E(l) = \emptyset$), a contradiction. 

Since each positive example $Y_j$ has $Y_j(x) = 0$, none of the positive examples can be in $T_E(r)$. Finally, since for every edge example $X_i$ in $T_E(v)$ that has $x$ as an endpoint has $X_i(x) = 1$, $X_i$ will land up in $T_E(r)$. This completes our proof.
\end{proof}

The following observation follows  from Observation~\ref{O:parsimonious}.
\begin{observation}\label{O:goodTree}
    Let $T$ be a \DT, for $E$, with root $v_0$, and let $v_{i-1}$ be a node of $T$ having $v_i$ as its left child. Moreover, let $v_0,v_1 \ldots,v_{i-1},v_i$ be the unique $(v_0,v_i)$-path in $T$ and each node $v_j$, for $j\in [i-1]$, is a vertex test node. Then, $T_E(v_i) = E \setminus (\bigcup_{j\in [i-1]}H(f(v_j)))$.
\end{observation}

Next, we have the following easy lemma that will be used to prove one side of our reduction. 

\begin{lemma}\label{L:Onesideboth}
    If $(G,k,p)$ is a Yes-instance of \textsc{PVC}, then there is a \DT $T$ such that $|T| = dep(T) \leq k$ and $|O(T,E)| \leq \ell \eta$.
\end{lemma}
\begin{proof}
    Let $S$ be a minimal partial vertex cover of $G$ of size at most $k$ that hits at least $p$ edges. Let $u_1,\ldots, u_k$ be an arbitrary ordering of vertices of $S$. Then, we construct a \DT $T$ (for $E$) in the following manner. See Figure~\ref{fig:oneSide} for a reference. The root node of $T$ is $v_1$. Moreover, for $i \in  [k-1]$, each node $v_i$ has a left child $v_{i+1}$ and a right child which is a negative leaf node $l_{i}$. The children of the node $v_k$ are both leaf nodes: the left child is a positive leaf node $l'$, and the right child is a negative leaf node $l_k$. Finally, for each test node $x_i$, let $f(v_i) = u_i$ and $\lambda(v_i) = 0$. Clearly, $|T|= dep(T) =|S|\leq k=s$. 

    Let $E'(G)\subseteq E(G)$ be the set of edges hit by $S$ and let $E''\subseteq E$ be the set of examples corresponding to the edges in $E'(G)$. Then, due to the construction of $T$ and Observation~\ref{O:goodTree}, we have that each leaf other than $l'$ (which is a positive leaf) receives only negative examples and is a negative leaf. So, the outliers are only coming from the leaf $l'$ (i.e., $O(T,E) \subseteq T_E(l')$), and these are exactly the negative examples in $E^- \setminus E''$ (due to Observation~\ref{O:goodTree}). Since each edge in $E(G)$ corresponds to exactly $\eta$ examples in $E$, we have $|O(T,E)| \leq |E^-|-|E''| = \eta|E(G)|-\eta|E'(G)| = \eta(m-p)= \ell\eta$. This completes our proof.
\end{proof}
\begin{figure}
    \centering
    \includegraphics[scale =1]{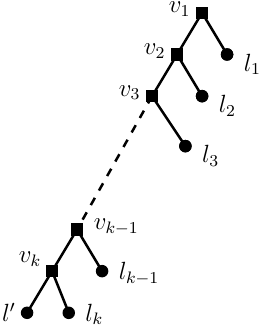}
    \caption{Illustration of $T$ corresponding to a set $S= \{u_1,\ldots, u_k\}$. Here, for $i\in [k]$, $v_i$ is a test node with $f(v_i)=u_i$ and $\lambda(v_i) = 0$. Moreover, the right child $l_i$ of $v_i$ is a negative leaf, and the left child $l'$ of $v_k$ is a positive leaf.}
    \label{fig:oneSide}
\end{figure}

Let $B_i$, for $i\in \eta$, be a block of $E$, and $T$ be a \DT for $E$. Then, we say that a (dummy) test node $v$ of $T$ \textit{intersects} $B_i$ if $f(v) =d_0$ and $2(i-1)m+1 \leq \lambda(v) \leq 2im$.  Observe that a (dummy) test node can intersect at most one block $B_i$. Thus, $T$ can have at most $|T|$ many blocks that are intersected by some test, and hence at least $\eta -|T|$ (we assume that $\eta \geq |T|$) blocks of $E$ are not intersected by any test node of $T$. Let $l'$ be a leaf of a \DT $T$ for $\mathcal{E}$. Moreover, we say that a block $B_i$ is \textit{contained} in $l'$ if each positive example $e$ of $B_i$ is in $T_{E}(l')$, i.e., $T_{E}(l') \cap E^+ \cap B_i = B_i \cap E^+$. Finally, we have the following observation. 
\begin{observation}\label{O:hardness}
    If a block $B_i$ is not intersected by any test node of $T$, then $B_i$ is contained in some leaf of $T$. Thus, there are at least $\eta -|T|$ blocks of $E$ that are contained in some leaf of $T$.
\end{observation}
\begin{proof}
 Recall that any two distinct positive examples of $E$ differ in exactly one feature, which is $d_0$. Targeting contradiction, assume that there are two distinct positive examples $e$ and $e'$ of some block $B_i$ such that $B_i$ is not intersected by any test node and $e$ and $e'$ end up in different leaves. Hence, there exists at least one test node $v$ such that $f(v) = d_0$ and $\lambda(v)$ is such that either $e(d_0) \leq \lambda(v)$ and $e'(d_0) >\lambda(v)$ or  $e'(d_0) \leq \lambda(v)$ and $e(d_0) >\lambda(v)$. WLOG, let us assume that $e(d_0) \leq \lambda(v)<e'(d_0)$. Thus, by the construction of $E$, we have that $2(i-1)m+1 \leq \lambda(v) < 2im$, which in turn implies that the test node $v$ intersects $B_i$, a contradiction. The rest of the proof follows from the fact that at least $\eta -|T|$ blocks of $E$ are not intersected by any test node of $T$. 
\end{proof}

Let $T$ be a \DT for $E$. Moreover, let $B$ be a block of $E$ such that $B$ is contained in some leaf $l$ of $T$. Then, let $N(l,B)$ denote the set of negative examples of block $B$ that are in $T_E(l)$, i.e., $N(l,B) = \{B\cap E^-  \cap T_{E}(l) \}$. Next, we have the following lemma.

 \begin{lemma}\label{L:HI1}
        If $|N(l,B)| > \ell$ for each block $B$ such that $B$ is contained in some leaf $l$ of $T$, then $|O(T,E)|\geq  (\eta -|T|) (\ell+1)$.
    \end{lemma}
    \begin{proof}
        Consider a leaf $l$ of $T$ such that the block $B$ is contained in $l$. First, observe that the number of outliers the block $B$ contributes to the leaf $l$ is at least $N(l,B)$. This is trivial if $l$ is a positive leaf, and if $l$ is a negative leaf, then observe that $B\cap E^+  = m \geq |N(l,B)|$. Thus, $B$ contributes at least $|N(l,B)|$ many outliers to $O(T,E)$. Finally, since we have at least $\eta -|T|$ blocks such that each of these blocks is contained in some leaf, due to Observation~\ref{O:hardness}, and every such block $B$ contributes at least $|N(l,B)| \geq \ell+1$ outliers to $T$, we have that $|O(T,E)|\geq  (\eta -|T|) (\ell+1)$.
    \end{proof}

The following lemma  establishes that if there is some block $B$ of $E$ such that $B$ is contained in some leaf $l$ of $T$ and $|N(l,b)| \leq \ell$, then $G$ has a partial vertex cover of size at most $dep(T)$ that hits at least $p = m-\ell$ edges.
\begin{lemma}\label{L:HI2}
    Let $T$ be a \DT for $E$. If there exists some block $B$ of $E$ such that $B$ is contained in some leaf $l$ of $T$ and $|N(l,B)| \leq \ell$, then $(G,dep(T),p)$ is a Yes-instance of \textsc{PVC}.
\end{lemma}
\begin{proof}
    Let $v_0$ be the root of $T$. Consider the unique $(v_0,l)$-path $P$ in $T$. We claim that the the set $S = \{ f(v)~|~ v\in V(P)\} \setminus \{d_0\}$ is a partial vertex cover of $G$ that hits at least $p$ edges. Since $B$ is not intersected by any dummy test node (due to Observation~\ref{O:hardness}), each negative example of $B$ that is not in $T_E(l)$ is hit by some vertex feature on the path $P$. Since $|N(l,B)| \leq \ell$, the number of negative examples of $B$ hit by vertex test nodes in $P$ is at least $m-\ell$. Since each negative example in $B$ corresponds to a unique edge of $G$, the set $S$ hits at least $m-\ell = m- (m-p) = p$ edges of $E(G)$. This completes our proof.
\end{proof}

\subsection{W[1]-hardness proofs.} Now, we present our main lemmas. The following lemma completes our reduction for \DTSO. 
\begin{lemma}\label{L:DTSO}
    Let $\eta = s(\ell+2)$. Then, $(G,k,p)$ is a Yes-instance of \textsc{PVC} iff $(E,s,t)$ is a YES-instance of \DTSO.
\end{lemma}
\begin{proof}
    The proof of one direction follows from Lemma~\ref{L:Onesideboth}. 

    In the other direction, let $(E,s,t)$ be a YES-instance of \DTSO, and let $T$ be a \DT such that $|O(T,E)| \leq t  =\eta \ell = s\ell(\ell+2)$ and $|T| =s$. Due to Lemma~\ref{L:HI1}, if $|N(l,B)| >\ell$ for each block $B$ of $E$ such that $B$ is contained in some leaf of $T$, then $|O(T,E)| \geq (\eta -|T|)(\ell+1)$. Since $|T| =s$ and $\eta = s(\ell+2)$, $|O(T,E)| \geq  (s(\ell+2)-s)(\ell+1) = s(\ell+2-1)(\ell+1) = s(\ell+1)^2$. Since $s(\ell^2+2\ell+1) > s(\ell^2+2\ell)  = \ell \eta = t$, it contradicts the fact that $(E,s,t)$ is a YES-instance of \DTSO. Hence, there is at least one block $B$ of $E$ such that $B$ is contained in some leaf $l$ of $T$ and $|N(l,B)|\leq \ell$. Therefore, due to Lemma~\ref{L:HI2} and the fact that $dep(T) \leq s = k$, $(G,k,p)$ is a Yes-instance of \textsc{PVC}.
\end{proof}

Next, we have the following lemma that completes our reduction for \DTDO.

\begin{lemma}\label{L:DTDO}
Let $\eta = 2^d(\ell+2)$. Then, $(G,k,p)$ is a Yes-instance of \textsc{PVC} iff $(E,d,t)$ is a YES-instance of \DTDO.
\end{lemma}
\begin{proof}
    The proof of one direction follows from Lemma~\ref{L:Onesideboth}. 

    In the other direction, let $(E,d,t)$ be a YES-instance of \DTDO, and let $T$ be a \DT such that $|O(T,E)| \leq t =\eta \ell = 2^d(\ell+2) \ell$ and $dep(T) =d$. Since $T$ is a binary tree, observe that $|T| \leq 2^d$.  Now, due to Lemma~\ref{L:HI1}, if $|N(l,B)| >\ell$ for each block $B$ of $E$ such that $B$ is contained in some leaf $l$ of $T$, then $|O(T,E)| \geq (\eta -|T|)(\ell+1)$. Since $|T| \leq 2^d$ and $\eta = 2^d(\ell+2)$, $|O(T,E)| \geq  (2^d(\ell+2)-2^d)(\ell+1) = 2^d(\ell+1)^2$. Since $2^d(\ell^2+2\ell+1) > 2^d(\ell^2+2\ell)  = \ell \eta = t$, it contradicts the fact that $(E,d,t)$ is a YES-instance of \DTDO. Hence, there is at least one block $B$ of $E$ such that $B$ is contained in some leaf $l$ of $T$ and $|N(l,B)|\leq \ell$. Therefore, due to Lemma~\ref{L:HI2} and the fact that $dep(T) \leq d = k$, we have that $(G,k,p)$ is a Yes-instance of \textsc{PVC}.
\end{proof}

Finally, we have the following theorem.
\WHard*

\begin{proof}
    Recall that in our construction of \CI $E$ from an instance $(G,k,t)$, $\delta_{max} \leq 3$. Moreover, we have that $d = s = k$.  Hence, the proof follows from our construction, Lemmas~\ref{L:DTSO} and \ref{L:DTDO}, and Preposition~\ref{P:pvcHard}.
\end{proof}

\section{\textsc{FPT} algorithms for \DTSO and \DTDO}
In this section, we complement Theorem~\ref{th:Whard} by establishing that \DTSO and \DTDO are \FPT when parameterized by $s+\delta_{max}+t$ and $d+\delta_{max}+t$, respectively. In regard to previous literature, we consider the size of our \CI $E$ as $||E||=|E|\cdot(|F(E)|+1)\cdot \log D_{max}$.  Our \FPT algorithm extends the \FPT algorithms for \DTD and \DTS parameterized by $d+\delta_{max}$ and $s+\delta_{max}$ provided by Eiben et al.~\cite{eibenLargeDomain}. The algorithm of~\cite{eibenLargeDomain} builds upon the algorithm of Ordyniak and Szeider~\cite{ordyniak}. Our algorithm follows the steps of their algorithm, and to avoid repetition, we provide details only about the parts that are new and non-trivial.
One important component of both these algorithms is that we can enumerate all the minimal support sets of size $k$ in \FPT (parameterized by $k+\delta_{max}$) time. 

Let $T$ be a \DT that classifies $E$. Then, observe that $F(T)$ is a support set for $E$~\cite{ordyniak}. We extend this notion to our setting. Consider $E$ and a subset $S\subseteq F(E)$. Let $E_1,\ldots,E_p$ be a partition of examples in $E$ such that for $i\in [p]$, each example in $E_i$ agrees on every feature in $S$. Moreover, for each $E_i$, $i\in j$, let $E^+_i$ denote the set of positive examples in $E_i$ and $E^-_i$ denote the set of negative examples in $E_i$. Then, we say that $E_i$ has $\min\{|E^+_i| ,|E^-_i|\}$ outliers for $S$, and we denote it by $B(S,E_i) = \min\{|E^+_i| ,|E^-_i|\}$. Finally, we say that $S$ is a \textit{support set with} $B(S,E) = \sum_{i\in [p]}B(S,E_i)$ \textit{outliers} for $E$. Our first easy observation is the following.
\begin{observation}\label{L:supportOurs}
    Let $T$ be a \DT such that $|O(T,E)| \leq t$. Then, $B(F(T),E) \leq t$, that is, $F(T)$ has at most $t$ outliers for $E$.
\end{observation}

Ordyniak and Szeider~\cite{ordyniak} established that we can enumerate all mininal support sets of size at most $k$ in time \FPT in $\delta_{max}$ and $k$. 
\begin{lemma}[\cite{ordyniak}]\label{L:enumerate}
    Let $E$ be a \CI and $k$ be an integer. Then, there is an algorithm that in time $\mathcal{O}(\delta_{max}^k |E|)$ enumerates all (of the at most $\delta_{max}^k$) minimal support sets of size at most $k$ for $E$.
\end{lemma}
We adapt Lemma~\ref{L:enumerate} to our setting in the following manner. We remark that this part is the most non-trivial extension and the rest of the proof roughly follows the algorithm from~\cite{eibenLargeDomain} with the only change being that instead of checking if $T$ classifies $E$, we check if $|O(T,E)|\leq t$.
\begin{lemma}\label{L:enumerateOurs}
    Let $E$ be a \CI and $k$ be an integer. Then, there is an algorithm that enumerates all (of the at most $\delta_{max}^k(t+1)^{2k}$) minimal  support sets allowing at most $t$ outliers for $E$ and of size at most $k$ in time $\mathcal{O}((\delta_{max}+2)^{k+t})|E|^{\mathcal{O}(1)}$.   
\end{lemma}
\begin{proof}
    We have the following branching-based algorithm. We keep two sets $S$ and $O$, representing the solution set and the set of outliers, respectively. Initially, we set $S=\emptyset$ and $O=\emptyset$. When $|S|>k$ or $|O|>t$, we backtrack. Else, if $B(S,E)\leq t$, then we report $S$ and backtrack. Else, pick some $\mathsf{p}\in E^+\setminus O$ and $\mathsf{n} \in E^-\setminus O$ such that $S$ does not distinguish $\mathsf{p}$ and $\mathsf{n}$; consider $\delta(\mathsf{p},\mathsf{n})+2$ branches: for each feature $f\in \delta(\mathsf{p},\mathsf{n})$, consider a branch where $S=S\cup \{f\}$, and consider two additional branches where in one add $\mathsf{p}$ to $O$ and in the other add $\mathsf{n}$ to $\mathsf{O}$. Validity of our algorithm follows as for any $\mathsf{p}\in E^+$ and  $\mathsf{n} \in E^-$, either at least one feature of $\delta_{max}$ should be in (the final) $S$ or at least one of $\mathsf{n}$ and $\mathsf{p}$ should contribute to $B(S,E)$. 
    Each branching step has at most $\delta_{max}+2$ branches, and the branching tree has at most $k+t$ depth. So, the number of leaves in our branching tree is at most  $(\delta_{max}+2)^{k+t}$. Since each branch can be implemented in poly-time, the algorithm takes  $\mathcal{O}((\delta_{max}+2)^{k+t})|E|^{\mathcal{O}(1)}$ time in total. 
\end{proof}

\Cref{L:enumerateOurs} establishes that we can enumerate all minimal support sets of size at most $k$ that allow at most $t$ outliers in \FPT (parameterized by $k+\delta_{max}+t$) time. Then, the remaining ingredient of our algorithm is the notion of maintaining $t$ outliers in each of their (\cite{ordyniak} and ~\cite{eibenLargeDomain}) definitions.  Moreover, in each check, instead of checking if $T$ is a valid \DT for $E$, we check if $|O(T,E)| \leq t$, which increases the dependency on the running time of their procedures by time \FPT in $t$. To avoid repetition, we use some of the terms from their paper in the sketch below without explicitly defining them. Hence, following their algorithm, we can first enumerate all minimal support sets of size at most $k$ ($k=s$ for \DTSO and $k=2^d$ for \DTDO) that allow at most $t$ outliers. Then, for each such support set $S$, we build all possible \DT \textit{patterns} $T$ such that $F(T) = S\cup \{\blacksquare\}$ of size at most $s$ (resp., $2^d$) for \DTSO (resp., \DTDO). Then, for each of these \DT patterns, we compute the \textit{branching sets} $(S,|T|)$. Finally, for each subset $S'$ of $(S,|T|)$ \textit{branching set} we construct the minimum sized (resp., minimum depth) \DT $T'$ with feature set $S' \cup F(T)$. Finally, we have the following theorem.

\FPTT*

\section{Kernelization Complexity}
Observation~\ref{th:HSIncompressible} indicates that it is unlikely for \DTS and \DTD to attain polynomial compressibility parameterized when $\delta_{max}$ is considered to be a part of the input. Since \HS admits a kernel of size at most $k^{\mathcal{O}(\Delta)}$, for example, using sunflower lemma~\cite{bookKernelization}, it becomes an interesting question to determine the kernelization complexity of \DTS and \DTD when $\delta_{max}$ is considered to be a constant (as $\delta_{max}=\Delta$ in the reduction). 

\subsection{Trivial Kernelization wrt $D_{max} + |F|$.}
Here, we establish that \DTS and \DTD admit a trivial polynomial kernel parameterized by $D_{max}+|F|$ when $\delta_{max}$ is  a fixed constant. The proof follows from simple arguments that upper bounds the number of examples in any classification instance by $(|F|D_{max})^{2\delta_{max}}$. 

\begin{lemma}\label{L:trivial}
    A classification instance can contain at most $2{|F|\choose \delta_{max}}(D_{max})^{\delta_{max}}$ examples.
\end{lemma}
\begin{proof}
    First, we establish an upper bound on the number of negative examples in $E$. Fix an arbitrary positive example $e$. (Recall that we assume that there is at least one positive and one negative example.) 
    Now, for each $S\subseteq F(E)$ such that $|S|=|F|-\delta_{max}$, let $E_S$ denote the set of negative examples in $E$ such that for each example in $e'\in E_S$ and $s\in S$, $e'(s) = e(s)$. Observe that this partitions the set of negative examples of $E$ into equivalence classes, and the number of these equivalence classes is at most ${|F| \choose |F|-\delta_{max}} = {|F| \choose \delta_{max}}$. Moreover, in each equivalence class $E_S$, each feature $s\in S$ can take a fixed value, and each feature $s'\in F(E)\setminus S$ can take one of the $D_{max}$ values. Hence the total number of distinct examples in $E_S$ can be at most $(D_{max})^{|F|-|S|} = (D_{max})^{\delta_{max}}$. Hence, the total number of negative examples in $E$ is at most $(D_{max})^{\delta_{max}}{|F| \choose \delta_{max}}$.

    Symmetrically, the total number of positive examples in $E$ is at most $(D_{max})^{\delta_{max}}{|F| \choose \delta_{max}}$. Hence, the total number of examples in $E$ is at most $2(D_{max})^{\delta_{max}}{|F| \choose \delta_{max}}$.
\end{proof}

Thus, Lemma~\ref{L:trivial} implies the following theorem. 
\TrivialKernel*
% \begin{theorem}\label{T:trivialPoly}
%  \DTS, as well as \DTD, parameterized by $D_{max}+|F|$ admit a trivial polynomial kernel when $\delta_{max}$ is a constant. 
% \end{theorem}

\subsection{Incompressibility for \DTD wrt $d+|F|$}
Next, we establish that even when $\delta_{max} =3$, \DTD parameterized by $d+|F|$ does not admit a polynomial compression, unless \NPoly. To this end, we provide a non-trivial AND-composition for \DTD parameterized by $|F|+d$ such that $\delta_{max} =3$. We will use the following construction.

\begin{figure}
    \centering
    \includegraphics[scale = 0.6]{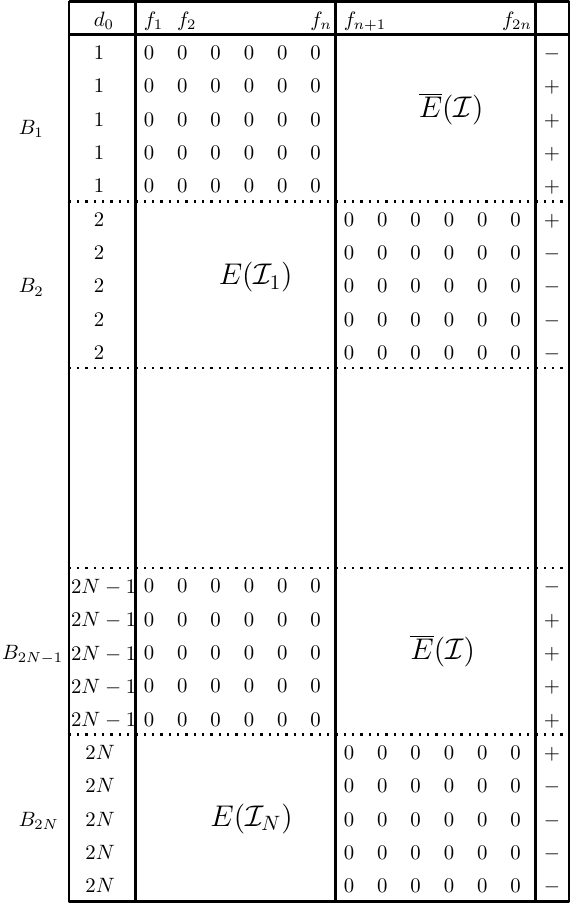}
    \caption{The construction of $E$ from $\overline{E}(\mathcal{I})$ and $E(\mathcal{I}_j)$, for $j\in [N]$.}
    \label{fig:incompressible}
\end{figure}
\medskip
\noindent\textbf{Construction.} An instance $(G,k)$ of \textsc{Vertex Cover} can be modeled as an instance $(\mathcal{F},U,k)$ of \HS with arity at most two such that $U=V(G)$ and $\mathcal{F} = E(G)$. Let $I_1 =(G_1,k),\ldots, I_N = (G_N,k)$ be $N$ instances of \textsc{Vertex Cover} such that for $i,j\in N$, $|V(G_i)|= |V(G_j)| = n$.  Assume WLOG that $V(G_i) = V(G_j)$. Moreover, assume that $N= 2^{\ell}$, for some $\ell \in \mathbb{N}$, as otherwise, we can add some dummy instances to make $N$ a  power of $2$. Furthermore, let $\mathcal{I}_1, \ldots, \mathcal{I}_N$ be the corresponding instances of \HS, each having arity at most 2, and let $E(\mathcal{I}_j)$ be the \CI corresponding to the \HS instance $\mathcal{I}_j$, for $j\in [N]$.  Moreover, let $F(E(\mathcal{I}_j))  = \{f_1,\ldots, f_n\}$.

Now, let $G$ be a fixed graph such that $|V(G)| = n$ and $G$ has a minimum vertex cover of size $k$ which is known to us. One such graph (assuming $n>k$) can be a split graph $G=(C,I)$ such that $|C|=k$ and each vertex in $I$ is connected to each vertex in $C$. Let $\mathcal{I}$ be the corresponding \HS instance of $G$. Consider the \CI $\overline{E}(\mathcal{I})$, and let $F(\overline{E}(\mathcal{I})) = \{f_{n+1}, \ldots, f_{2n}\}$. 

Finally, we create the \CI $E$ as follows. See Figure~\ref{fig:incompressible} for an illustration. Let $F(E) =  \{f_1,\ldots f_{2n}\} \cup \{d_0\}$, i.e., $F(E)$ contains each feature in $F(\overline{E}(\mathcal{I}))$, each feature in $F(E(\mathcal{I}_j))$ for $j\in [N]$, and a \textit{dummy feature} $d_0$. Now, $E$ will contain $2N$ \textit{blocks} of examples $B_1,\ldots, B_{2N}$. For $i\in [2N]$, the examples in $B_i$ will depend on whether $i$ is odd or even as follows. 
\begin{enumerate}
    \item $i$ is odd. In this case, for each positive/negative example $e'$ in $\overline{E}(\mathcal{I})$, we add a positive/negative example $e$ to $B_i$ as follows. First, set $e(d_0) = i$. Second, for $j\in [n]$, set $e(f_j) = 0$. Finally, for $n+1\leq j \leq 2n$, set $e(f_j) =e'(f_j)$. Observe that when restricted to the features in $\{f_{n+1},\ldots, f_{2n}\}$, $B_i$ is identical to $\overline{E}(\mathcal{I})$, and for every other feature of $E$, the examples in $B_i$ have identical values. Hence, the examples in $B_i$ can be classified using a \DT of size (resp., depth) at most $k$, and cannot be classified using a \DT of size (resp., depth) less than $k$. 

    \item $i$ is even. Let $p = \frac{i}{2}$. In this case, for each positive/negative example $e'$ of $E(\mathcal{I}_p)$, we add a positive/negative example $e$ to $B_i$ in the following manner. First, set $e(d_0) = i$. Second, for $j\in [n]$, set $e(f_j) = e'(f_j)$. Finally, for $n+1\leq j \leq 2n$, set $e(f_j) = 0$. Similarly to the previous case, observe that when restricted to the features in $\{f_1,\ldots, f_n\}$, $B_i$ is identical to $E(\mathcal{I}_p)$, and for every other feature, the examples in $B_i$ have identical values. Hence, examples in $B_i$ can be classified using a \DT of depth (resp., size)  $\ell$ iff $E(\mathcal{I}_p)$ can be classified using a \DT of depth (resp., size) $\ell$.
\end{enumerate}

Finally, let $p = \log_2 N+1$ and $d = p+k = \log_2 N+k+1$. This completes our construction. Observe that in each block $B_i$, for $i\in [2N]$, there is exactly one example $e$ such that for every vertex feature $f_j$, $j\in [2n]$, $e(f_j) = 0$. Moreover, $e$ is a positive example if $i$ is odd, and a negative example if $i$ is even. We refer to such examples as \textit{dummy examples}. All other examples are said to be \textit{edge examples}. We have the following observations about our \CI $E$. 

\begin{observation}\label{OI:trivial}
  Let each $I_j$, for $j\in [N]$, be a Yes-instance of \textsc{Vertex Cover}. Then, the examples of block $B_i$, for $i\in [2N]$, can be classified using a \DT $T_i$ of depth at most $k$. Moreover, when $i$ is odd, the examples of $B_i$ cannot be classified using a \DT of depth less than $k$.   
\end{observation}

Let $T$ be a \DT for $E$ that classifies $E$. Then, a test node $v$ of $T$ is said to be a \textit{dummy test node} if $f(v) = \{d_0\}$ and is said to be a \textit{vertex test node}, otherwise. Moreover, similar to the reduction in W[1]-hardness proofs, we assume that for each test node $v$ with children $l$ and $r$, $T_E(v)$ is non-uniform and $T_E(l) \neq \emptyset$ and $T_E(r) \neq \emptyset$. We discuss the effect of these test nodes below.

\medskip
\noindent\textbf{Dummy test node.} Let $v$ be a dummy test node of $T$, and let $l$ and $r$ be the left and right child of $v$, respectively. Then, $f(v)= d_0$ and $\lambda(v) \in [2N]$. Observe that for each example $e\in T_E(v)$, if $e(d_0) \leq \lambda(v)$, then $e\in T_e(l)$, else, $e\in T_e(r)$. Moreover, since the value of $d_0$ is the same for each example of a block $B$, all examples of $B$ that are in $T_E(v)$ will end up in either $T_E(l)$ or $T_E(r)$, i.e. either $T_E(v)\cap B = T_E(l) \cap B$ or $T_E(v) \cap B = T_E(r) \cap B$. Hence, we have the following.

\begin{observation}\label{O:dummyTest}
    Let $v_0$ be the root of a \DT $T$ for $E$. Moreover, let $v$ be a dummy test node of $T$ with $l$ and $r$ as the left and right child of $v$, respectively. Furthermore, let every node on the unique $(v_0,v)$-path be a dummy test node. Then, if $T_E(l)$ (resp., $T_E(r)$) contains some example of a block $B$, then $T_E(l)$ (resp., $T_E(r)$) contains every example of the block $B$.
\end{observation}

\medskip
\noindent\textbf{Vertex test node.} Let $v$ be a vertex test node of $T$ with $l$ and $r$ as the left and the right child of $v$, respectively. Then, $f(v) \in \{f_1, \ldots, f_{2n}\}$ and $\lambda(v) = 0$ (as otherwise, either $T_E(l) = \emptyset$ or $T_E(r) = \emptyset$). Moreover, if $f(v)\in \{f_1,\ldots, f_n\}$, then $r$ is a negative leaf; else, $r$ is a positive leaf.
Hence, we have the following. 
\begin{observation}\label{OI:1}
    Let $v$ be a vertex test node in a \DT $T$ for $E$. Moreover, let $l$ and $r$ be the left and right child of $v$, respectively. Then, $\lambda(v)=0$ and $r$ is a negative (resp., positive) leaf if $f(v)\in \{f_1,\ldots, f_n\}$ (resp., $f(v) \in \{f_{n+1,\ldots, f_{2n}}\}$). 
\end{observation}

Now, we have the following lemma, which proves one side of our reduction.

\begin{lemma}\label{LI:oneSide}
    If each $I_j$, for $j\in [N]$, is a Yes-instance of \textsc{Vertex Cover}, then $(E,d)$ is a Yes-instance of \DTD.
\end{lemma}
\begin{proof}
    Let each $I_j = (G_j,k)$ be a Yes-instance of \textsc{Vertex Cover}. Then, we have that each instance $\mathcal{I}_j$, for $j\in [N]$, is a Yes-instance of \HS, and hence, due to Proposition~\ref{P:HS}, each $(E(\mathcal{I}_j),k)$, for $j\in [N]$, is a Yes-instance of \DTD. Similarly, since $(G,k)$ is a Yes-instance of \textsc{Vertex Cover}, we have that $(\mathcal{I},k)$ is a Yes-instance of \HS, and hence, due to Proposition~\ref{P:HS}, $(\overline{E}(\mathcal{I}),k)$ is a Yes-instance of \DTD. %Let $p=\log_2(N)$.

    Now, consider a \DT $T'$ such that $T'$ is a complete binary tree of depth $p$ (recall that $p=\log_2N+1$), and each test node of $T'$ is a dummy test node. The thresholds of the test nodes of $T'$ are defined in the following top-down manner. For root $v_0$ of $T'$, $\lambda(v_0) = \frac{\max(T'_E(v_0)) -\min(T'_E(v_0)) +1}{2} = N$. Now, let $v$ be a test node of $T'$ with parent $p$ such that $\lambda(p)$ is known. Then, $\lambda(v) = \frac{\max(T'_E(v)) -\min(T'_E(v)) +1}{2}$. Observe that $T'$ has $2^{p}=2N$ leaves. Let these leaves be $l_1,\ldots, l_{2N}$ in the left to right ordering. It is not difficult to see that $T'_E(l_i)= B_i$, for $i\in [2N]$. Now, since each of $B_i$ can be classified using a \DT $T_i$ of depth at most $k$ (Observation~\ref{OI:trivial}), we replace each leaf $l_i$ of $T'$ with $T_i$ to get a \DT $T$ that classifies $E$. Clearly, $dep(D) \leq p+k=d$. This completes our proof. 
\end{proof}

In the following, we prove that if $(E,d)$ is a Yes-instance of \DTD, then $I_j$, for $j\in [N]$, is a Yes-instance of \textsc{Vertex Cover}. Let $v$ and $v'$ be test nodes of \DT{s} $T$ and $T'$, respectively (possibly, $T=T'$ and $v = v'$). Then, we say that $v\sim v'$ if $f(v)= f(v')$ and $\lambda(v) = \lambda(v')$. We have the following straightforward observation.
\begin{observation}\label{O:subset}
    Let $u$ and $v$ (resp., $u'$ and $v'$) be the nodes of a \DT $T$ (resp., $T'$) such that 
    \begin{enumerate}
        \item $u$ (resp., $u'$) is an ancestor of $v$ (resp., $v'$) in $T$ (resp., $T'$),
        \item $T'_E(u') \subseteq T_E(u)$, and
        \item for each vertex $x$ on the unique $(u,v)$-path in $T$, there is a vertex $y$ on the unique $(u',v')$-path in $T'$ such that $x\sim y$.
    \end{enumerate}
    Then, $T'_E(v') \subseteq T_E(v)$. 
\end{observation}

Next, we have the following crucial lemma.
\begin{lemma}\label{LI:toBottom}
    Let $T'$ be a \DT that classifies $E$. Then, there is a \DT $T$  that classifies $E$,  $dep(T) \leq dep(T')$, and none of the dummy test nodes in $T'$ has a vertex test node as an ancestor.  
\end{lemma}
\begin{proof}
    Let $T$ be a \DT  that classifies $E$, $dep(T) \leq dep(T')$, and $T$ minimizes the number of vertex test nodes that have some dummy test node as a descendent. We claim that no dummy test node of $T$ has a vertex test node as an ancestor. Targeting contradiction, let there be a dummy test node $v$ in $T$ that has a vertex test node $u$ as an ancestor. Due to Observation~\ref{OI:1}, we know that the right child $r$ of $u$ is a leaf and let $l$ be the left child of $u$. Now, consider the tree $T_l$. Let $l_1,\ldots, l_\ell$ be the leaf nodes in $T_l$.  We construct the tree $T''$ from $T$ as follows (see Figure~\ref{fig:DTD}).

    First, we construct $T''_l$ from $T_l$ by replacing each leaf $l_i$, for $i\in [\ell]$, by a vertex test node $u_i$ such that $f(u_i) = f(u)$, $\lambda(u_i) = \lambda(u)=0$, $l_i$ is the left child of $u_i$ and add a right child $r_i$ of $u_i$. Now, we replace $T_u$ in $T$ with $T''_l$ to get the \DT $T''$. More formally, if $u$ is the root of $T$, then $l$ is the root of $T''$. Else, if $u$ is the left (resp., right) child of its parent $p$ in $T$, then $l$ is the left (resp., right) child of $p$ in $T''$. It is easy to see that $T''$ has the same depth as $T$ (since $T_u$ and $T''_l$ have the same depth and the rest of the tree is the same). It remains to prove that $T''$ classifies $E$. Observe that except for the leaf $r$ of $T$, every leaf of $T$ corresponds to a unique leaf in $T''$. Moreover, the leaf $r$ of $T$ corresponds to exactly $\ell$ leaves $r_1, \ldots, r_\ell$ of $T''$. Now, we have the following claim.
      \begin{claim}\label{CI:1}
          $T''$ classifies $E$.
      \end{claim}
      \begin{proofofclaim}
        Observe that $T_E(u) = T''_E(l)$. Targeting contradiction, suppose that $T''$ does not classify $E$. Then, there is a leaf $l'$ in $T''$ such that $l'$ is a positive (resp., negative) leaf, and there is a negative (resp., positive) example $e$ such that $e\in T''_E(l')$. Let $v_0$ and $v''_0$ be the root nodes of $T$ and $T''$, respectively.

        First, let $l'\notin \{ r_1, \ldots, r_\ell\}$. Then, recall that there is a unique leaf $l'$ in $T$ as well that corresponds to $l'$ in $T''$. Moreover, observe that for every test node $y$ on the unique $(v_0,l')$-path in $T$, there is a test node $y''$ on the unique $(v''_0,l')$-path in $T''$ such that $y \sim y''$ and $T_E(v_0) = T''_E(v''_0) =E$. Therefore, due to Observation~\ref{O:subset}, we have that $T''_E(l') \subseteq T_E(l')$, and hence $e\in T_E(l')$. Therefore, $T$ also misclassifies $e$, a contradiction to the fact that $T$ classifies $E$.

        Finally, let $l' =r_i$ for some $i\in [\ell]$. Similarly to the above case, we can establish here as well that $T''_E(r_i) \subseteq T_E(r)$ (since for every test node $y$ on the unique $(v_0,r)$-path in $T$, there is a test node $y''$ on the unique $(v''_0,r_i)$-path in $T''$ such that $y \sim y''$, and $T_E(v_0) = T''_E(v''_0) =E$). Therefore, if $e\in T''_E(r_i)$, then $e\in T_E(r)$, a contradiction to the fact that $T$ classifies $E$. 
        This completes the proof of our claim. 
      \end{proofofclaim}
      Since $T''$ classifies $E$ (Claim~\ref{CI:1}), $dep(T'') \leq dep(T)$, and $T''$ contains at least one less vertex test node that has some dummy test node as its descendent than $T$, we reach a contradiction. Thus, none of the dummy test nodes in $T$ has a vertex test node as its ancestor. 
\end{proof}
\begin{figure}
    \centering
    \includegraphics[scale=0.8]{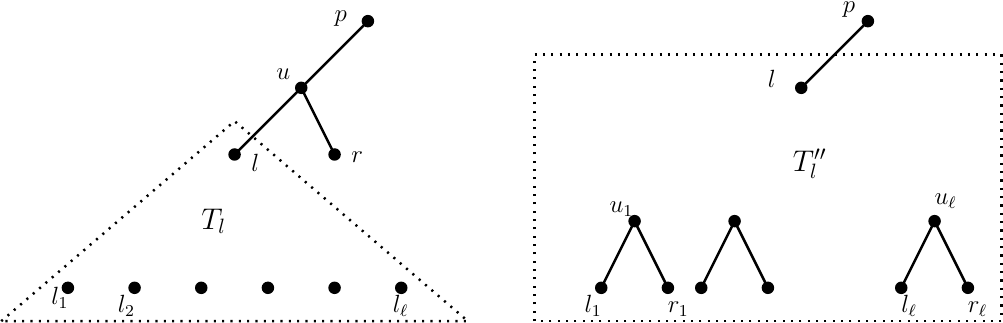}
    \caption{The construction of $T''_l$ from $T_u$.}
    \label{fig:DTD}
\end{figure}

Next, we have the following lemma.
\begin{lemma}\label{LI:otherSide}
  If $(E,d)$ is a Yes-instance of \DTD, then $I_j$, for $j\in [N]$, is a Yes-instance of \textsc{Vertex Cover}. 
\end{lemma}
\begin{proof}
    Let $T$ be a \DT with root $v_0$ such that $dep(T) \leq d$ and $T$ classifies $E$. Due to Lemma~\ref{LI:toBottom}, we can assume that none of the dummy test nodes has a vertex test node as its ancestor. Let $T'$ be a subtree of $T$ such that each node in $T'$ except $v_0$ has a dummy test node as its parent. We note that each leaf of $T'$ can be either a vertex test node in $T$ or a leaf node in $T$. Also, note that for each node $v$ of $T'$, $T_E(v) = T'_E(v)$. We will now discuss the properties of $T'$.

    \begin{claim}\label{CI:O1}
        For each leaf $l$ of $T'$, $T'_E(l)=B_i$, for some $i\in [2N]$. Moreover, for each $i\in [2N]$, there is some leaf $l$ of $T'$ such that $T'_E(l)=B_i$.
    \end{claim}
    \begin{proofofclaim}
        Let $l$ be a leaf of $T'$, and $p$ be the parent of $l$. (We can assume that $l$ has a parent, as otherwise, $T$ becomes an empty tree.) Observe that $l$ is either a leaf node or a vertex test node in $T$ such that none of the descendent of $T$ is a dummy test node. Since we assume that for each node $v$ of $T$, $T_E(v) \neq \emptyset$, there is some example $e \in T'_E(l)$. Let $e\in B_i$ for some $i\in [2N]$. Since each test node on the unique $(v_0,p)$-path in $T'$ (as well as in $T$) is a dummy test node, due to Observation~\ref{OI:1}, every example of the block $B_i$ is contained in $T'_E(l)$, i.e., $B_i \subseteq T'_E(l)$.  

        Now, to complete our proof, we only need to show that for every $j\neq i$, $j\in [2N]$, and $e' \in B_j$, it holds that $e'\notin T'_E(l)$. Targeting contradiction, let there be some $j\neq i$ such that $j\in [2N]$, $e' \in B_j$, and $e'\in T'_E(l)$. Then, due to Observation~\ref{OI:1}, $B_j \subseteq T'_E(l)$. WLOG assume that $i<j$. Then, we claim that $B_{i+1}$ (possibly $i+1=j$) is also contained in $T'_E(l)$. This is easy to see since for every dummy test node $v$ on the unique $(v_0,l)$-path, if $\lambda(v)> e(d_0) =i$ and $\lambda(v) > e'(d_0)=j$ (resp., $\lambda(v)\leq e(d_0) = i$ and $\lambda(v) \leq e'(d_0)=j$), then $\lambda(v) > i+1$ (resp., $\lambda(v) \leq i+1$) as well. Hence, $B_{i+1} \in T'_e(l)$. Therefore, $T'_E(l)$ contains two examples $e_i\in B_{i}$ and $e_{i+1}\in B_{i+1}$ such that for each vertex feature $f_j$, for $j\in [2n]$, $e_i(f_j) = e_{i+1}(f_j) = 0$, and $e_i$ is a positive example and $e_{i+1}$ is a negative example. 
        
        Let $l'$ is a leaf of $T$ such that $e_i\in T_E(l')$. Since each test node in the unique $(l,l')$-path is a vertex test node and $e_i$ and $e_{i+1}$ agree on every vertex feature, $e_{i+1} \in T_E(l')$. Hence $T_E(l')$ is non-uniform, which contradicts the fact that $T$ classifies $E$. This completes proof of our claim.
    \end{proofofclaim}

    Next, we claim that $T'$ is a complete binary tree of depth $p = \log_2(N)+1$.
    \begin{claim}\label{CI:O2}
        $T'$ is a complete binary tree and $dep(T') = p$.
    \end{claim}
    \begin{proofofclaim}
        Due to Claim~\ref{CI:O1}, since each $B_i$, for $i\in [2N]$, is contained in a unique leaf of $T'$, we know that $T'$ has at least $2N$ leaves. Therefore, it is easy to see that $dep(T') \geq p$, as any binary tree of depth at most $p-1 = \log_2(N)$ can have at most $N$ leaves.

        Next, we establish that $dep(T') \leq p$. Targeting contradiction, let $dep(T') \geq p+1$ and let $l$ be a leaf of $T'$ that maximizes the number of (dummy) test nodes in the $(v_0,l)$-path. Moreover, let $p$ be the parent of $l$, and $l'$ be the other child of $p$. Observe that $l'$ is also a leaf of $T'$. Moreover, let $B_i =T'_E(l)$ and $B_j = T'_E(l')$. We claim that either $i$ is odd or $j$ is odd. Note that $B_i \cup B_j =  T_E(p)$. Now, it can be established using arguments used in the proof of Claim~\ref{CI:O1} that $i$ and $j$ are consecutive numbers, i.e., either $i=j+1$ or $i=j-1$. Therefore, either $i$ or $j$ is odd. 

        WLOG assume that $i$ is odd. Since $T_E(l) = B_i$ and $T$ classifies $E$, the subtree $T_l$ should classify $B_i$. It follows from Observation~\ref{OI:trivial} that the depth of $T_l$ is at least $k$ since $i$ is odd and $T_l$ classifies $B_i$. Let $l''$ be a leaf of $T_l$ that maximizes the number of test nodes in the $(l,l'')$-path in $T_l$. Then, $dep(T)$ is at least the number of test nodes in $(v_0,l)$-path in $T'$ plus the number of test nodes in the $(l,l'')$-path in $T$. Therefore, $dep(T) \geq p+1 +k >d$, which is a contradiction since $dep(T) \leq d$. This establishes that $dep(T') = p$. 

        Finally, since there are $2N = 2^{p}$ blocks in $E$ and each block is contained in a unique leaf of $T'$, we have that $T'$ has at least $2^p$ leaves. Since the depth of $T'$ is $p$, it is only possible if $T'$ is a complete binary tree. 
    \end{proofofclaim}

    Finally, we have the following claim.
    \begin{claim}\label{CI:main}
        Each $I_j = (G_j,k)$, $j\in [N]$, is a Yes-instance of \textsc{Vertex Cover}. 
    \end{claim}
    \begin{proofofclaim}
        Targeting contradiction, let there be some $j\in [N]$ such that $(G_j,k)$ is a No-instance of \textsc{Vertex Cover}, and hence, $(\mathcal{I}_j,k)$ is a No-instance of \HS. Therefore, due to Proposition~\ref{P:HS}, $(E(\mathcal{I}_j),k)$ is a No-instance of \DTD. Hence, for every \DT $T_j$ such that $T_j$ classifies $E(\mathcal{I}_j)$, $dep(T_j)>k$. Moreover, recall that $T_j$ classifies $E(\mathcal{I}_j)$ iff $T_j$ classifies the examples of $B_j$.

        Now, let $l$ be the leaf of $T'$ such that $T'_E(l) = B_j$.  Since the subtree $T_l$ of $T$ classifies examples of $B_j$, we have that $dep(T_l) >k$. Moreover, since $T'$ is a complete binary tree of depth $p$ (Claim~\ref{CI:O2}), we have that $dep(T) > p+k=d$, a contradiction to the fact that $dep(T)=d$. This completes the proof of our claim.
    \end{proofofclaim}

    Therefore, due to Claim~\ref{CI:main}, we have that each $I_j = (G_j,k)$, $j\in [N]$, is a Yes-instance of \textsc{Vertex Cover}. This completes our proof.
\end{proof}

Finally, our construction and Lemmas~\ref{LI:oneSide} and \ref{LI:otherSide} imply that \DTD is AND-compositional. Moreover, since $\lambda_{max} \leq 3$ for $E$, $|F| = 2|V(G_i)|+1$, and $d = \log_2 N+k$, we have the following theorem as a consequence.
\begin{theorem}
    \DTD parameterized by $d+|F|$ is AND-compositional. Hence, \DTD parameterized by $d+|F|$ does not admit a polynomial compression even when $\delta_{max} \leq 3$, unless \NPoly.
\end{theorem}

\section{Conclusion}
We considered a generalization of \DTL where the constructed \DT is allowed to disagree with the input \CI on a ``\textit{few}'' examples. We considered \DTDO and \DTSO which are generalizations of \DTD and \DTS, respectively, from both parameterized and kernelization perspectives. \DTSO and \DTDO consider additive errors in accuracy. As pointed out by~\cite{DTGeometry}, it may be desirable to design algorithms to respect Pareto-optimal trade-offs between size and measures of accuracy. In this regard,  \Cref{th:Whard,th:FPT} imply that some of the existing  \FPT algorithms for \DTS and \DTD cannot be directly generalized, and one needs to consider the number of misclassified examples as a parameter.  One interesting question in this direction can be if including number of features instead of $\delta_{max}$ as a parameter remove the dependence on $t$, i.e., are \DTSO and \DTDO \FPT parameterized by $s+|F|$ and $d+|F|$, respectively.   Moreover, one  interesting open question in regard of kernelization complexity of \DTS is the following.

\begin{question}
    Does \DTS parameterized by $s+|F|$ admit a polynomial kernel when $\delta_{max}$ is a fixed constant?
\end{question}

To establish the incompressibility of \DTD in \Cref{th:dtdIncompressible}, we use a \DT of depth $\log N+1$ to filter out blocks of examples; each leaf of this ``filter tree'' gets a unique block. The size of this filter tree is $2N$, and hence our construction directly cannot be used to prove the same result for \DTS parameterized by $s+|F|$ (since the parameter cannot have linear dependency on the number of instances in AND-composition).

Komusiewicz et al.~\cite{komusiewicz2023computing} recently considered the parameterized complexity of two generalizations of \DTS where the task is to design $\ell$ \DT{s} such that for each example, the majority of the trees agree with it. For $\ell=1$ both these generalizations model \DTS. They provided \FPT algorithms parameterized by $s+\delta_{max}+\ell+D_{max}$ and asked as an open question if these problems admit a polynomial kernel parameterized by the same parameters. Notice that Observation~\ref{th:HSIncompressible} answers this question negatively since the incompressibility is implied for these problems even by fixing $\ell=1$ and $D_{max}=2$. Moreover, note that since the trivial kernel in Theorem~\ref{th:trivial} is obtained by establishing an upper bound on the size of the \CI, the kernelization result applies to their generalizations as well.

Sometimes, in practice, it might be desirable that some examples are \textit{marked} ineligible to be outliers. For example, while examples from some \textit{noisy} part of the data are allowed to be outliers, an example from the \textit{robust} part of the data must  never be misclassified. We can model this setup using \DTSO and \DTDO by making $t+1$ copies of each robust example if we allow $E$ to be a multiset. Otherwise, if we need $E$ to be a set, we think that our algorithm can be adjusted to restrict outliers to be only from the ``noisy'' part of the data.

Since the appearance of the conference version of this paper, there have been some interesting developments. Staus et al.~\cite{staus2025witty} designed Witty, a fast and practical solver to compute minimum-sized decision trees, which is an heuristic-augmented implementation of an \FPT algorithm of Komusiewicz et al.~\cite{komusiewicz2023computing}.  Moreover, it is worth mentioning that Witty can solve \DTSO (with some extra overhead).

\section{Acknowledgments}
Meirav Zehavi acknowledges the support by European Research Council (ERC) project titled PARAPATH (101039913) and the ISF grant with number ISF--1470/24. Harmender Gahlawat acknowledges the support by the IDEX-ISITE initiative CAP 20-25 (ANR-16-IDEX-0001), the International Research Center ``Innovation Transportation and Production Systems'' of the I-SITE CAP 20-25, the ANR project GRALMECO (ANR-21-CE48-0004), and the ERC grant PARAPATH.

%%%%%%%%%%%%%%%%%%%%%%%%%%%%%%%%%%%%%%%%%%%%%%%%%%%%%%%%%%%%%%%%%%%%%%%%

%%% The next two lines define, first, the bibliography style to be 
%%% applied, and, second, the bibliography file to be used.

 \bibliographystyle{plainurl} 
% \bibliography{sample}

\bibliography{main}

\end{document}